\documentclass[12pt]{article}

\usepackage{amsbsy,amsmath,amsthm,amssymb}
\usepackage{graphics}
\usepackage[demo]{graphicx}
\usepackage{setspace}
\usepackage{natbib}
\usepackage{fullpage}
\usepackage{url}
\usepackage{multirow}
\usepackage{color}
\usepackage{subcaption}
\usepackage{times}
\usepackage[demo]{graphicx}
\usepackage{natbib}
\usepackage{algorithm}
\usepackage{multirow}
\usepackage{amsmath}
\usepackage{graphics}
\usepackage{setspace}
\usepackage{natbib}
\usepackage{xcolor}
\usepackage{soul}

\usepackage{fullpage}
\usepackage{url}
\usepackage{multirow}
\usepackage{color}

\theoremstyle{plain}
\newtheorem{theorem}{Theorem}

\newtheorem{lemma}{Lemma}

\theoremstyle{definition}
\theoremstyle{remark}

\newcommand{\Real}{\mathbb{R}}
\DeclareMathOperator{\tr}{tr}

\DeclareMathOperator{\VEC}{vec}
\DeclareMathOperator{\sign}{sign}

\DeclareMathOperator{\Sp}{\mathbb{S}}

\DeclareMathOperator{\Cs}{\mathcal{C}}

\DeclareMathOperator{\card}{card}

\DeclareMathOperator{\Sft}{Sft}

\newcommand{\argmin}{\qopname\relax m{arg\,min}}

    {

\pdfminorversion=4

\begin{document}

\def\spacingset#1{\renewcommand{\baselinestretch}%
{#1}\small\normalsize} \spacingset{1}

\begin{center}
\LARGE \bf Estimating Multiple Precision Matrices with Cluster Fusion Regularization
\end{center}
\begin{center}
Bradley S. \textsc{Price}, Aaron J. \textsc{Molstad},
and Ben \textsc{Sherwood}
\end{center}

\newcommand{\foo}{\thefootnote}
\renewcommand{\thefootnote}{\relax}
\footnotetext{Bradley S. Price, Management Information Systems Department, West Virginia University
    (E-Mail: \emph{brad.price@mail.wvu.edu)}.
    Aaron J. Molstad, Department of Statistics, University of Florida
    (E-mail: \emph{amolstad@ufl.edu}).
    Ben Sherwood, School of Business, University of Kansas
    (E-mail: \emph{ben.sherwood@ku.edu}).}

%

\begin{abstract}
We propose a penalized likelihood framework for estimating multiple precision matrices from different classes. Most existing methods either incorporate no information on relationships between the precision matrices, or require this information be known \emph{a priori}. The framework proposed in this article allows for simultaneous estimation of the precision matrices and relationships between the precision matrices, jointly. Sparse and non-sparse estimators are proposed, both of which  require solving a non-convex optimization problem. 
To compute our proposed estimators, we use an iterative algorithm which alternates between a convex optimization problem solved by blockwise coordinate descent and a $k$-means clustering problem.
Blockwise updates for computing the sparse estimator require solving an elastic net penalized precision matrix estimation problem, which we solve using a proximal gradient descent algorithm. We prove that this subalgorithm has a linear rate of convergence. In simulation studies and two real data applications, we show that our method can outperform competitors that ignore relevant relationships between precision matrices and performs similarly to methods which use prior information often uknown in practice.

\end{abstract}

\doublespace
\section{Introduction}
Many applications in statistics and machine learning require the estimation of multiple, possibly related, precision matrices. For example, to perform classification using quadratic discriminant analysis (QDA), a practitioner must estimate two or more precision matrices, e.g., see Chapter 4 of \citet{friedman2001elements}. 
Similarly, it is often of scientific interest to estimate multiple graphical models when the same variables are measured on subjects from multiple classes, e.g., \citet{JointEstimationMGM}.


In this work, the data $(x_1,y_1),\ldots, (x_n,y_n)$ are assumed to be a realization of $n$ independent copies of the random pair $(X,Y)$ such that 
$Y$ has support  $\Cs=\{1,\ldots, C\}$ and 
$$(X \vert Y=c) \sim {\rm N}_p\left(\mu_{*c},\Omega_{*c}^{-1}\right), \quad c \in \Cs,$$ where $\mu_{*c} \in \mathbb{R}^p$ and $\Omega_{*c} \in \mathbb{S}^{p}_+$ are unknown, and $\mathbb{S}^{p}_+$ denotes the set of $p \times p$ symmetric, positive definite matrices.  Let $n_c=\sum_{i=1}^n 1(y_i=c)$ be the sample size
for the $c$th class, 
let $\bar x_c= n_{c}^{-1} \sum_{i=1}^n x_i 1(y_i=c)$ be the observed sample mean for the $c$th class, and let
$$
S_c=\frac{1}{n_c}\sum_{i=1}^n (x_i-\bar{x}_c)(x_i-\bar{x}_c)^T 1(y_i=c), \qquad c\in\Cs,
$$
be the sample covariance matrix for the $c$th class. Define ${\bf{\Omega}}=\left\{\Omega_1,\ldots,\Omega_C\right\}$. After profiling over the means and class probabilities, negative two times the log-likelihood is
\begin{equation} \label{eq:gee}
g({\bf{\Omega}}) = \sum_{c\in \Cs} n_c\{\tr(S_c\Omega_c)-\log\det(\Omega_c)\}.
\end{equation} 
A natural estimator of $\Omega_{*c}$, when it exists, is the maximum likelihood estimator $S_c^{-1}$. 
In settings where the maximum likelihood estimator does not exist, e.g., when $p > n_c$, a practitioner could instead estimate the $\Omega_{*c}$'s separately using penalized normal maximum likelihood \citep{Pour11,fan2016overview}. 
Sparsity inducing penalties are especially popular in penalized normal maximum likelihood \citep{yuan_07,dAspermont_08,glasso,SPICE,NEWGlasso} because a zero in the $(j,k)$th entry of $\Omega_{*c}$ implies the conditional independence of the $j$th and $k$th variables in the $c$th class. However when the precision matrices are similar across classes, e.g., when the $\Omega_{*c}$'s share sparsity patterns, jointly estimating the $\Omega_{*c}$'s can be more efficient than methods that estimate each precision matrix separately.

Many methods exist for estimating multiple precision matrices under the assumption of shared sparsity patterns across classes. For example, \citet{JointEstimationMGM} proposed a hierarchical penalty which encourages zeroes in the same entries of estimates of the $\Omega_{*c}$'s.  Similarly, \citet{danaher_2014} proposed the fused graphical lasso estimator (FGL)
$$ \argmin_{\Omega_c \in \mathbb{S}^p_+, c \in \Cs} \left\{ g({\bf{\Omega}}) + \lambda_1 \sum_{c \in \Cs}\|\Omega_c\|_1 + \lambda _2 \hspace{-8pt} \sum_{(j,k) \in \Cs \times \Cs}\hspace{-8pt}\|\Omega_j - \Omega_k\|_1 \right\},$$
where $\| A \|_1=\sum_{j=1}^p\sum_{k=1}^p |A_{j,k}|$. The first FGL penalty, controlled by positive tuning parameter $\lambda_1$, promotes elementwise sparsity separately \textit{within} classes. The second penalty, controlled by $\lambda_2$, promotes elementwise equality jointly \textit{across} classes. For sufficiently large values of the tuning parameter $\lambda_2$, the FGL estimates of the $\Omega_{*c}$'s will have exactly equivalent sparsity patterns. \citet{price_2015} proposed a computationally efficient alternative to FGL, called \textit{ridge fusion} (RF), which used squared Frobenius norm penalties in place of the $L_1$ norm penalties in FGL.  \citet{price_2015} also investigated FGL and RF as methods for fitting the QDA model. These approaches for fitting the QDA model are related to \citet{friedman_1989}, who proposed the regularized discriminant analysis (RDA) estimator of the precision matrices. The RDA approach estimates multiple precision matrices for QDA using a linear combination of the sample covariance matrices for each class and the pooled sample covariance matrix across all the classes.  \citet{bilgrau_2015} generalized the work of \citet{price_2015} to estimate multiple precision matrices sharing a common target matrix.  

Joint estimation procedures such as those proposed by \citet{danaher_2014} and \citet{price_2015} can perform well when all $C$ precision matrices are similar. However, there are settings when FGL and RF may perform unnecessary or inappropriate shrinkage. 
Notice, the second part of the ridge fusion penalty proposed by \citet{price_2015} can be rewritten as
$$
\sum_{(j,k) \in \Cs \times \Cs} \|\Omega_j-\Omega_k\|_F^2=C\sum_{l \in \Cs} \|\Omega_l-\bar{\Omega}\|_F^2, \quad \text{ where } \quad \bar{\Omega}=\frac{1}{C}\sum_{c \in \Cs} \Omega_c
$$
and $\|A\|_F^2 = {\rm tr}(A'A)$ is the squared Frobenius norm. This formulation suggests that FGL and RF can be viewed as shrinking all precision matrices towards a common precision matrix. Regularization of this type may be problematic if there are substantial differences in the population precision matrices across classes. For instance, consider the case that there are two \textit{groupings} (i.e., clusters) of the $C$ classes denoted $D_1$ and $D_2$, where $D_1 \cap D_2$ is empty and $D_1 \cup D_2 = \Cs$. Suppose the $(j,k)$ entry of $\Omega_{*c}$, $\left[\Omega_{*c}\right]_{j,k} = 0$ for $c \in D_1$, but $\left[\Omega_{c}\right]_{j,k} \neq 0$ for $c \in D_2$ for all $j \neq k$. This type of scenario may occur, for example, when the variables are the expression of $p$ genes belonging to some pathway and the classes represent certain disease subtypes. Two subtypes may have similar gene-gene dependence, which are distinct from the another subtype (e.g., controls). In these settings, FGL may perform poorly since sparsity patterns are only shared within a subset of classes.  If such clusters were known \emph{a priori}, it may be preferable to apply FGL or RF to the groups separately, but when groupings are unknown, they must be estimated from the data. 

Methods such as those proposed by \citet{zhu2014} and \citet{saegusa2016} address this issue of FGL and RF. The structural pursuit method proposed by \citet{zhu2014} allows for heterogenity in precision matrices using the truncated lasso penalty of \citet{shen12} to promote elementwise equality and shared sparsity patterns across pre-defined groups of precision matrices.  The method proposed by \citet{saegusa2016}, known as LASICH, allows for heterogeneity of precision matrices through the use of a graph Laplacian penalty which incorporates prior information about how different classes' sparsity patterns are related. Since such prior information is often not available in practice, the authors propose the HC-LASICH method: a two-step procedure which first uses hierarchical clustering to estimate relationship between precision matrices, then uses this estimate to apply the LASICH procedure.  In somewhat related work, \citet{Ma_2016} proposed the joint structural estimation method to use prior information on shared sparsity patterns in a two step procedure that first estimates the shared sparsity pattern and then estimates the precision matrices based on the shared sparsity constraints.  More recently, \citet{Jalali_19} extended the work of \citet{Ma_2016} to the case where prior information on edge relationships not need be known.  This is done using a Bayesian approach that incorporates a multivariate Gaussian mixture distribution on all possible sparsity patterns.

In this article, we propose a penalized likelihood framework for simultaneously estimating the $C$ precision matrices {and how the precision matrices relate to one another}. Like FGL and RF, our method can exploit the similarity of precision matrices belonging to a group, but avoids the unnecessary shrinkage of FGL or RF when groups differ. Unlike LASICH, the proposed method does not require any prior information about the relationships between the classes, nor does it require clustering to take place before estimation of the precision matrices. 
We study the use of our estimator for application in quadratic discriminant analysis and Gaussian graphical modeling in settings where there are groupings of classes which share common dependence structures.  Computing our estimator is nontrivial since the penalized objective function we minimize is discontinuous. To overcome this challenge, we propose an iterative algorithm, in which we alternate between updating groupings and updating precision matrix estimates.  As part of our algorithm for the sparse estimator we propose (see Section 2), we must solve an elastic net penalized precision matrix estimation problem. To do so, we propose a graphical elastic net iterative shrinkage thresholding algorithm (GEN-ISTA). We prove this GEN-ISTA has a linear convergence rate and characterize the set to which the solution belongs. 

\section{Joint estimation with cluster fusion penalties}

\subsection{Methods}

Define $(D_1,\ldots, D_Q)$ to be an unknown $Q$ element partition of the set $\Cs$. For convenience, we will refer to $D_q$ as the $q$th cluster.  Let $\lambda_1 > 0$, $\lambda_2 \geq 0$, and the positive integer $Q$ be user defined tuning parameters.

For any set $B$ define $\card{(B)}$ as the cardinality of $B$. The first estimator we will investigate is the \emph{cluster ridge fusion estimator} (CRF), which is defined as
\begin{equation}
\label{crf}
({\bf{\hat{\Omega}}_{CRF}},\hat{D}) = \argmin_{\Omega_c \in \Sp^p_+, c \in \Cs, D_1,\ldots,D_Q} \left\{ g({\bf{\Omega}}) + \frac{\lambda_1}{2}\sum_{c\in\Cs} \|\Omega_c\|_F^2
 +\frac{\lambda_2}{2}\sum_{q=1}^Q\frac{1}{\card{(D_q)}} \sum_{c,m\in D_q} \|\Omega_c -\Omega_{m}\|_F^2\right\}.
 \end{equation} 
We refer to the penalty associated with $\lambda_2$ as the \textit{cluster fusion penalty}, which promotes similarities in precision matrices that are
in the same cluster. The ridge fusion method (RF) proposed by \citet{price_2015} can be viewed as a special case of \eqref{crf} when $Q=1$.

We also propose a sparsity inducing version of the estimator, the \emph{precision cluster elastic net} (PCEN), which is defined as
\begin{equation}
\label{pcen}
({\bf{\hat{\Omega}}_{PCEN}},\tilde{D})=\argmin_{\Omega_c \in \Sp^p_+, c \in \Cs, D_1,\ldots,D_Q} \left\{ g({\bf{\Omega}}) + \lambda_1\sum_{c\in\Cs} \|\Omega_c\|_1
 +\frac{\lambda_2}{2}\sum_{q=1}^Q\frac{1}{\card{(D_q)}} \sum_{c,m\in D_q} \|\Omega_c -\Omega_{m}\|_F^2\right\}.
 \end{equation} 
When $\lambda_2 = 0$, ${\bf{\hat{\Omega}}_{PCEN}}$ is equivalent to estimating the $C$ precision matrices separately with $L_1$ penalized normal maximum likelihood using the same tuning parameter for each matrix. 
In our proposed estimators, the cluster fusion penalty is used to promote similarity in precision matrices that are in the same cluster, while estimating precision matrices in different clusters separately from one another. When estimating Gaussian graphical models, PCEN promotes elementwise similarity between precision matrices in the same cluster, in turn promoting similar sparsity patterns within the same cluster.  This differs from other methods, e.g., FGL proposed by \cite{danaher_2014}, which penalize the absolute value of entrywise differences across all precision matrices. 

Unlike the FGL fusion penalty, the squared Frobenius norm fusion penalty will not lead to exact entrywise equality between estimated precision matrices -- even those belonging to the same cluster. However, the Frobenius norm penalty facilitates fast computation and more importantly, an efficient search for clusters using existing algorithms for $k$-means clustering. 

Cluster fusion regularization was first proposed in the context in univariate response linear regression by \cite{witten_2014} to detect and promote similarity in effect sizes.  More recently \cite{price_sherwood} used cluster fusion regularization in multivariate response linear regression to detect and promote similarity in fitted values. 

If $D_1,\ldots,D_Q$ were known, then \eqref{crf} and \eqref{pcen}
could be rewritten as 
\begin{eqnarray}
\label{crf_dfix}
{\bf{\tilde{\Omega}}_{CRF}}=\argmin_{\Omega_c \in \Sp^p_+, c \in \Cs,} \left\{ g({\bf{\Omega}}) +\frac{ \lambda_1}{2}\sum_{c\in\Cs} \|\Omega_c\|_F^2
 +\frac{\lambda_2}{2}\sum_{q=1}^Q\frac{1}{\card{(D_q)}} \sum_{c,m\in D_q} \|\Omega_c -\Omega_{m}\|_F^2\right\},\\
\label{pcen_dfixed}
{\bf{\tilde{\Omega}}_{PCEN}}=\argmin_{\Omega_c \in \Sp^p_+, c \in \Cs} \left\{ g({\bf{\Omega}}) + \lambda_1\sum_{c\in\Cs} \|\Omega_c\|_1
 +\frac{\lambda_2}{2}\sum_{q=1}^Q\frac{1}{\card{(D_q)}} \sum_{c,m\in D_q} \|\Omega_c -\Omega_{m}\|_F^2\right\}.
\end{eqnarray}
The optimization in \eqref{crf_dfix} can be identified as $Q$ separate ridge fusion estimation problems \citep{price_2015}.  The optimization in \eqref{pcen_dfixed} is also separable over the $Q$ clusters, and in Section \ref{pcen_alg} we propose a block coordinate descent algorithm to solve \eqref{pcen_dfixed}. 

In the Supplementary Material, we describe a validation likelihood based approach to select the tuning parameters for use in CRF and PCEN; and further discuss a more computationally efficient information criterion which can be used for tuning parameter selection for PCEN. 

\section{Computation}

\subsection{Overview}
The objective functions in \eqref{crf} and \eqref{pcen} are not continuous and non-convex with respect to $D_1,\ldots,D_Q$ because changing cluster membership results in discrete changes in the objective function. However, in \eqref{crf_dfix} and \eqref{pcen_dfixed}, $D_1,\ldots,D_Q$ are fixed so that the objective functions are strictly convex and convex, respectively, with respect to $\bf{\Omega}$. 
To compute both CRF and PCEN, we propose an algorithm that iterates between solving for $D_1,\ldots,D_Q$ with $\bf{\Omega}$ fixed, and then solving for $\bf{\Omega}$ with $D_1,\ldots,D_Q$ fixed.  This procedure is similar to those of \citet{witten_2014} and \citet{price_sherwood}. We describe both algorithms in the following subsections. 

\subsection{Cluster ridge fusion algorithm}
\label{sec:crfa}
Assume $\lambda_1 >0$, $\lambda_2 \geq 0$, and $Q \in \mathbb{Z}_+$ are fixed where $\mathbb{Z}_+$ denotes the set of positive integers.  We propose the following iterative algorithm to solve \eqref{crf}: 

\begin{enumerate}
\item Initialize ${\bf{\tilde{\Omega}}_{CRF}}^1=\{\tilde{\Omega}_1^{1},\ldots, \tilde{\Omega}_C^1\}$ as a set of diagonal matrices where the $j$th diagonal element of $\tilde{\Omega}_k^{1}$ is $(S_{k_{jj}})^{-1}$. 
\item For the $w$th step where $w>1$ repeat the steps below until the estimates for $\hat{D}^{w-1}_1,\ldots,D^{w-1}_Q$ are equivalent to $\hat{D}^w_1,\ldots, \hat{D}^w_Q$.  
\begin{enumerate}
\item Holding $\bf{\tilde{\Omega}}^{w-1}$ fixed, obtain the $w$th iterate of $(\hat{D}_1,\ldots,\hat{D}_Q)$ with
\begin{equation}
\label{rf_clust}
(\hat{D}_1^{w},\ldots,\hat{D}_Q^{w})=\argmin_{D_1,\ldots,D_Q}  \left\{ \sum_{q=1}^Q\frac{1}{\card{(D_q)}} \sum_{c,m\in D_q} \|\tilde{\Omega}^{w-1}_c -\tilde{\Omega}^{w-1}_{m}\|_F^2\right\}.
\end{equation}

This is equivalent to solving the well studied $k$-means clustering optimization problem on $C$ $p^2$ vectors \citep{witten_2014}. 

\item Holding $(\hat{D}_1^{w},\ldots,\hat{D}_Q^{w})$ fixed at the $w$th iterate, obtain the $w$th iterate of the precision matrices with
\begin{equation}
\label{rf-update}
{\bf{\tilde{\Omega}}_{CRF}}^w=\argmin_{\Omega_c \in \Sp^p_+, c \in \Cs} \left\{ g({\bf{\Omega}}) + \frac{\lambda_1}{2}\sum_{c\in\Cs} \|\Omega_c\|_F^2
 +\frac{\lambda_2}{2}\sum_{q=1}^Q\frac{1}{\card{(\hat{D}^w_q)}} \sum_{c,m\in \hat{D}^w_q} \|\Omega_c -\Omega_{m}\|_F^2\right\}.
\end{equation} 
This is identical to the optimization in \eqref{crf_dfix} and can be solved with $Q$ parallel RF estimation problems, with the $q$th objective function taking the form 

$$
\sum_{c \in D_q} \left\{ n_c\left(\tr(S_c\Omega_c) -\log\det(\Omega_c)\right) +\frac{\lambda_1}{2} \|\Omega_c\|_F^2\right\} +\frac{\lambda_2}{2\card{(D_q)}}\sum_{c,m \in D_q} \|\Omega_l-\Omega_m\|_F^2.  
$$
\end{enumerate}
\end{enumerate}


To protect against $k$-means clustering update in 2(a) from selecting a local optima, our implementation uses 100 random starts, and selects the clustering which gives the lowest objective function value \citep{hartigan_1979, krishna_1999}.  


\subsection{Precision cluster elastic net algorithm}
\label{pcen_alg}
For the PCEN estimator, we propose to use the same iterative procedure as in Section \ref{sec:crfa}.  The algorithm iterates between
a $k$-means clustering algorithm to obtain the estimated clusters and a blockwise coordinate descent algorithm which uses the \textit{graphical elastic net iterative thresholding algorithm} (GEN-ISTA) to obtain new iterates of the precision matrices at each iteration.  


Again, let $\lambda_1>0$, $\lambda_2 \geq 0$ and $Q \in \mathbb{Z}_+$ be fixed. Formally, the iterative algorithm is as follows: 

\begin{enumerate}
\item Initialize ${\bf{\tilde{\Omega}_{PCEN}}}^1=\{\tilde{\Omega}_1^{1},\ldots, \tilde{\Omega}_C^1\}$ as a set of diagonal matrices where the $j$th diagonal element  of  $\tilde{\Omega}_k^{1}$ is $(S_{k_{jj}})^{-1}$. 
\item For the $w$th step where $w>1$ repeat the steps below until the estimates for $\tilde{D}^{w-1}_1,\ldots,\tilde{D}^{w-1}_Q$ are equivalent to $\tilde{D}^w_1,\ldots, \tilde{D}^w_Q$.  
\begin{enumerate}
\item Holding $\bf{\tilde{\Omega}_{PCEN}}^{w-1}$ fixed, obtain the $w$th iterate for $\tilde{D}$ with
\begin{equation}
\label{pcen_clust}
(\tilde{D}_1^{w},\ldots,\tilde{D}_Q^{w})=\argmin_{D_1,\ldots,D_Q}  \sum_{q=1}^Q\frac{1}{\card{(D_q)}} \sum_{c,m\in D_q} \|\tilde{\Omega}^{w-1}_c -\tilde{\Omega}^{w-1}_{m}\|_F^2.
\end{equation}

\item Holding $(\tilde{D}_1^{w},\ldots,\tilde{D}_Q^{w})$ fixed at the $w$th iterate, obtain the $w$th iterate of the precision matrix estimates with

\begin{equation}
\label{pcen-update}
{\bf{\tilde{\Omega}}_{PCEN}}^w=\argmin_{\Omega_c \in \Sp^p_+, c \in \Cs} \left\{ g({\bf{\Omega}}) + \lambda_1\sum_{c\in\Cs} \|\Omega_c\|_1
 +\frac{\lambda_2}{2}\sum_{q=1}^Q\frac{1}{\card{(\tilde{D}^w_q)}} \sum_{c,m\in \tilde{D}^w_q} \|\Omega_c -\Omega_{m}\|_F^2\right\}.
\end{equation} 
\end{enumerate}
\end{enumerate}

Just as in the CRF Algorithm, to protect against selecting a local optima in the $k$-means clustering update in 2(a), our implementation uses 100 random starts, and selects the clustering which gives the lowest objective function value \citep{hartigan_1979, krishna_1999}. 

The update in  \eqref{pcen-update} is a non-trivial convex optimization problem.  As noted previously, \eqref{pcen-update} can be separated into $Q$ separate optimization problems, where the $q$th optimization can be written as 
\begin{equation}
\label{pcen_single}
\argmin_{\Omega_c \in \Sp_p^+, c \in \Cs} \left( \left[ \sum_{c \in D_q} n_c\{\tr(S_c\Omega_c)-\log\det(\Omega_c)\} \right] +\lambda_1\sum_{c\in D_q} \|\Omega_c\|_1
 +\frac{\lambda_2}{2\card{(\tilde{D}^w_q)}} \sum_{c,m\in \tilde{D}^w_q} \|\Omega_c -\Omega_{m}\|_F^2 \right).
\end{equation}

Since the $D_q$'s  are fixed, we propose to solve \eqref{pcen_single} using blockwise coordinate descent where each $\Omega_c$ is treated as a block.  That is, for each $D_q$, we update one $\Omega_c$ for $c \in D_q$ with all other $\Omega_{c'}$ for $c' \in D_q$ held fixed. The objective function for the $\Omega_c$ blockwise update, treating all other $\Omega_{c'}$, $c' \neq c$, $c, c' \in D_q$ as fixed is
\begin{equation}
\label{block_cd}
 n_c \left({\rm tr}\left[\left\{S_c - \frac{\lambda_2}{n_c {\rm card}(D^w_q)}\left( \sum_{c' \neq c} \Omega_{c'} \right) \right\}  \Omega_c\right] - \log {\rm det}(\Omega_c) \right) + \lambda_1 \|\Omega_c\|_1 + \frac{\lambda_2({\rm card}(D^w_q) - 1)}{2{\rm card}(D^w_q)}\|\Omega_c\|_F^2.
\end{equation}

Define 
$$\tilde{S}_c=\left\{S_c - \frac{\lambda_2}{{\rm n_c card}(D^w_q)}\left( \sum_{c' \neq c} \Omega_{c'}\right)  \right\}, \quad 
\gamma_{c1}=\frac{\lambda_1}{n_c},\quad \gamma_{c2}= \frac{\lambda_2({\rm card}(D^w_q) - 1)}{2n_c{\rm card}(D^w_q)},$$
so that the argument minimizing \eqref{block_cd} can be expressed 
\begin{equation}
\label{reform}
\argmin_{\Omega_c \in \Sp_p^+}  \left\{ {\rm tr}(\tilde{S}_c \Omega_c) - \log\det( \Omega_c) + \gamma_{c1} \|\Omega_c\|_1 + \gamma_{c2} \|\Omega_c\|_F^2 \right\},
\end{equation}
which can be recognized as the elastic net penalized normal likelihood precision matrix estimator. To compute \eqref{reform}, we propose the GEN-ISTA, an elastic net variation of the algorithm proposed by \citet{rolfs12}, called G-ISTA, which was used to solve problems like \eqref{reform} with $\gamma_{c2}=0$.  Iterative shrinkage thresholding algorithms (ISTA), are a special case of the proximal gradient method, which are commonly used to solve penalized likelihood optimization problems.  We refer the reader to \citet{beck09} and \citet{polson2015proximal} for more on iterative shrinkage thresholding algorithms and proximal algorithms, respectively.

Our approach uses a first order Taylor expansion to derive a majorizing function of the objective in \eqref{reform}.   Let
$$ f(\Omega_c) \equiv  {\rm tr}(\tilde{S}_c \Omega_c ) - \log\det( \Omega_c) + \gamma_{c2}\|\Omega_c\|_F^2$$ 
and let $\tilde{\Omega}_c$ be the previous iterate of $\Omega_c$. Because $\nabla f$ is Lipschitz over compact sets of $\mathbb{S}_+^p$, we have that
\begin{equation}\label{eq:majorizer1}
f(\Omega_c) \leq f(\tilde{\Omega}_c) + {\rm tr}[(\Omega_c - \tilde{\Omega}_c)'\nabla f(\tilde{\Omega}_c) ] + \frac{1}{2t}\|\Omega_c - \tilde{\Omega}_c\|_F^2,
\end{equation}
for sufficiently small step size $t$, with equality when $\Omega_c = \tilde{\Omega}_c$. Thus, we can majorize $f$ with the right hand side of \eqref{eq:majorizer1}: using this inequality and the fact that $\nabla f(\Omega_c) = \tilde{S}_c - \Omega_c^{-1} + 2 \gamma_{2c} \Omega_c $, we have
\begin{equation} 
\label{majorize}
 f(\Omega_c)  \leq - \log {\rm det}(\tilde{\Omega}_c)  + {\rm tr}[\tilde{\Omega}_c(\tilde{S}_c+\gamma_{c2}\tilde{\Omega}_c)] + {\rm tr}[(\Omega_c - \tilde{\Omega}_c)'(\tilde{S}_c - \tilde{\Omega}_c^{-1}+2\gamma_{c2}\tilde{\Omega}_c) ] + \frac{1}{2t}\|\Omega_c - \tilde{\Omega}_c \|_F^2.
 \end{equation}
 Letting $g_t(\Omega_c; \tilde{\Omega}_c)$ denote the right hand side of \eqref{majorize}, it follows that for sufficiently small $t$, 
 \begin{equation}  \label{eq:majorizer2}
 f(\Omega_c) + \gamma_{c1}\|\Omega_c\|_1 \leq g_t(\Omega_c; \tilde{\Omega}_c) + \gamma_{c1}\|\Omega_c\|_1,
 \end{equation}
so that at $\tilde{\Omega}_c$, the right hand side of \eqref{eq:majorizer2} is a majorizer of \eqref{reform}. Thus, to solve \eqref{reform}, we use an iterative procedure: given the previous iterate $\tilde{\Omega}_c$, we construct $g_t(\Omega_c; \tilde{\Omega}_c)$, then we minimize $g_t(\Omega_c; \tilde{\Omega}_c) + \gamma_{c1}\|\Omega_c\|_1$ to obtain the new iterate. This choice of majorizer is convenient since the new optimization problem simplifies to the proximal operator for the $L_1$ norm since
$$ \argmin_{\Omega_c \in \mathbb{S}^p}\left\{ g_t(\Omega_c; \tilde{\Omega}_c) + \gamma_{c1}\|\Omega\|_1  \right\} = \argmin_{\Omega_c \in \mathbb{S}^p} \left\{ \frac{1}{2}\|\Omega_c - Z_{c,t}\|_F^2 + t \gamma_{c1}\|\Omega_c\|_1 \right\},$$
where $Z_{c,t} = t(\tilde{S}_c - \tilde{\Omega}^{-1}_c+2\gamma_{c2}\tilde{\Omega}_c)$ and $\mathbb{S}^p$ denotes the set of $p \times p$ symmetric matrices. In the following subsection, we will show that there always exists a step size such that the solution to the proximal operator above is positive definite, and hence, the iterates remain feasible for \eqref{reform}. 



To summarize, we propose the GEN-ISTA, which updates from iterate $k$ to iterate $k+1$ with
\begin{align}
\Omega_c^{(k+1)}
& = \argmin_{\Omega \in \mathbb{S}^p}  \left\{\frac{1}{2}\|\Omega_c - {\Omega}^{(k)}_c + t\{ \tilde{S}_c - (\Omega^{(k)}_c)^{-1}+2\gamma_{c2}\Omega^{(k)}_c\} \|_F^2 + t \gamma_{c1} \|\Omega_c\|_1\right\}\notag \\
&=\Sft(\left({\Omega}^{(k)}_c  - t\{\tilde{S}_c - \Omega^{-1(k)}_c+2\gamma_{c2}\Omega^{(k)}_c\}, t \gamma_{c1}\right),\label{eq:prox_update}
\end{align}
where for a $p \times p$ matrix $A$ and $\eta>0$, $\Sft(A,\eta)$ is the elementwise soft thresholding operator such that $\left[\Sft(A,\tau)\right]_{ij}=\sign(A_{ij})\max(|A_{ij}|-\tau,0)$.  To select $t$ for use in \eqref{eq:prox_update}, we use a backtracking line search. 
For the step to be accepted we check that the the condition in \eqref{majorize} is met and that $\Omega_c^{(k+1)}\in \Sp_+^p$. If both conditions are not met, a smaller step size $t$ must be used.  In Section \ref{ista-convergence} we show that for a pre-specified $t$, which is a function of $\tilde{S}_c,\gamma_{c2},$ and $p$, that this update will always be contained in $\Sp^p_+$. 


Formally we propose the GEN-ISTA algorithm with backtracking to solve \eqref{reform}:
\begin{enumerate}
\item Initialize, $k=0$, $\eta \in (0,1)$, $\epsilon>0$, $t_0 > 0$, and $\Omega_c^{(0)}\in \Sp^p_+$.
\item While $|f(\Omega_c^{(k)})-f(\Omega_c^{(k+1)})|>\epsilon$ or $k<1$
\begin{enumerate}
\item Set $t=t_0$
\item $\Omega_c^{(k+\frac{1}{2})}={\Omega}^{(k)}_c  - t\{ \tilde{S}_c - (\Omega_c^{(k)})^{-1}+2\gamma_{c2}\Omega^{(k)}_c\}$
\item $\Omega_c^{(k+1)}=\Sft(\Omega_c^{(k+\frac{1}{2})},t\gamma_{c1})$
\item If $\Omega_c^{(k+1)} \not\in \Sp^p_+$, then update $t=t\eta$ and return to Step 2 (b). Else, continue to Step 2 (e)
\item If 
$$
f(\Omega_c^{(k+1)})> f(\Omega_c^{(k)}) +\tr\left\{(\Omega_c^{(k+1)}-\Omega_c^{(k)})^T(\tilde{S}_c-(\Omega_c^{(k)})^{-1}+2\gamma\Omega_c^{(k)})\right\}+\frac{1}{2t}\|\Omega_c^{(k+1)}-\Omega_c^{(k)}\|_F,
$$
then update $t=t\eta$ and return to Step 2 (b). Else, continue to Step 2 (f).
\item Update $k=k+1$, and return to Step 2 (a).
\end{enumerate}
\end{enumerate}


As previously mentioned, the G-ISTA algorithm proposed by \cite{rolfs12} is a special case of the GEN-ISTA algorithm, when $\gamma_{c2}=0$, but there are substantial differences. In particular, \cite{rolfs12} only consider the case where $\tilde{S}_c$ is a symmetric, non-negative definite matrix, but in our application there is no guarantee that $\tilde{S}_c$ is non-negative definite.  In Section \ref{ista-convergence} we demonstrate the role of $\gamma_{c2}$ in the rate of convergence and the choice of appropriate step size, $t$. The elastic net penalized normal likelihood precision matrix estimation problem was also studied by \citet{atchade_19}, who proposed a stochastic gradient descent algorithm for solving \eqref{reform} with $p$ very large and $\tilde{S}_c$ being a sample covariance matrix. 

\subsection{Convergence Analysis of GEN-ISTA Algorithm} 
\label{ista-convergence}
In this section we will discuss the convergence of the GEN-ISTA subroutine proposed in the previous section. Our approach to convergence analysis is based on that of \citet{rolfs12}, but in our application, we must address that the input matrix $\tilde{S}_c$ may be indefinite. We show that despite the generality of the input matrix, our proximal gradient descent scheme is guaranteed to converge at a linear rate and that the maximum step size is a function of $\alpha$ and $\gamma_{c2}$, both of which are known in our blockwise coordinate descent scheme. Specifically, we show that there exists a worst case contraction constant, $\delta \in (0,1)$, such that
$$
\|\Omega_c^{(k+1)}-\Omega_c^{*}\|_F\leq \delta\|\Omega_c^{(k)}-\Omega_c^{*}\|_F.
$$ 
In our case $\delta$ is a function of $\tilde{S}_c, \gamma_{c2},$ and $p$. We will show that as $\gamma_{c2}$ increases, $\delta$ approaches 0. Throughout this section, for a $p \times p$ matrix $A$, let $\rho_1(A)\geq \rho_2(A) \geq \ldots \geq \rho_p(A)$ denote the ordered eigenvalues of $A$. 



Let $\Omega_c^{*}$ be the solution to \eqref{reform}. We first will show that $\Omega_c^{*}$ is contained in a compact subset of $\mathbb{S}^p_+$. 


\begin{lemma}
\label{optim_bounds}
Let $\gamma_{c1}>0$, $\gamma_{c2}>0$ and $\Omega_c^{*}$ to be the solution to \eqref{reform} then  $\alpha I \preceq \Omega_c^{*} \preceq \beta I$, where
$$
\alpha^{-1}=.5\left(\rho_{1}(\tilde{S}_c)+\gamma_{c1}p+\sqrt{(\rho_{1}(\tilde{S}_c)+\gamma_{c1}p)^2+8\gamma_{c2}}\right)
$$ 
and 
$$
\beta^{-1}=.5\left(\rho_{p}(\tilde{S}_c)-\gamma_{c1}p+\sqrt{(\rho_{p}(\tilde{S}_c)-\gamma_{c1}p)^2+8\gamma_{c2}}\right).
$$
\end{lemma}
The proof of Lemma \ref{optim_bounds} is contained in the Appendix and uses the dual formulation of \eqref{reform}. 

Our bounds are distinct from those in \citet{rolfs12} as theirs do not allow for $\tilde{S}_c$ which is indefinite. Notably, the $\alpha$ we obtain is the same as that in \citet{atchade_19}, although the $\beta$ we obtain is distinct, again owing to the indefiniteness of $\tilde{S}_c$. Next, we establish that the Lipschitz continuity of the gradient of \eqref{reform}, which we used to construct the majorizing function \eqref{eq:majorizer1}. 





\begin{lemma}
\label{lip_grad}
Assume $\alpha I\preceq \Omega_A,\Omega_B\preceq \beta I$ such that $0<\alpha<\beta<\infty$ then 
$$
\|\triangledown f(\Omega_A)-\triangledown f(\Omega_B)\|_F \leq  \sqrt{p}\left(\frac{1}{\alpha^2}+ 2\gamma_{c2}\right)\|\Omega_A-\Omega_B\|_F.
$$
Hence, $\bigtriangledown f(\Omega)= \tilde{S}_c-\Omega^{-1}+2\gamma_{c2}\Omega$ is Lipschitz 
on any compact subset of $\Sp_+^p$ with constant  $\sqrt{p}\left(\frac{1}{\alpha^2}+ 2\gamma_{c2}\right)$.
\end{lemma}

The proof of Lemma \ref{lip_grad} can be found in the Appendix. The combination of Lemma \ref{optim_bounds} and Lemma \ref{lip_grad} give us necessary and sufficient conditions to apply Theorem 3.1 of \cite{beck09} to \eqref{reform} to obtain a sublinear convergence rate between iterates of the objective function.

 Next, we present a lemma that ensures that there always exists a step size parameter $t$ such that the iterates of the algorithm are contained in a compact subset of $\Sp^p_+$.
\begin{lemma}
\label{iterate_bound}
Let $\gamma_{c1},\gamma_{c2}>0$ and define $\alpha$ and $\beta$ to be defined as presented in Lemma \ref{optim_bounds}.  Assume $t\leq \frac{a^2}{2\alpha^2\gamma_{c2}+1}$ then the iterates of the proposed algorithm satisfy $\alpha I \preceq \Omega_c^{(k)} \preceq b'I$ for all $k$ where 
$$
b'=\|\Omega_c^{*}\|_2+\|\Omega_c^{(0)}-\Omega_c^{*}\|_F\leq \beta+\sqrt{p}(\beta-\alpha).
$$
\end{lemma}

The result of Lemma \ref{iterate_bound} is similar to those presented in \cite{rolfs12} and \cite{atchade_19}. The proof can be found in the Appendix.

Finally, we present a result on the linear convergence rate for our algorithm given the iterates are contained on on a compact subset of $\Sp^p_+$. 
\begin{theorem}
\label{our_alg_conv}
Let $\alpha$ and $\beta$ be defined the same as in Lemma \ref{optim_bounds}.  Then for constants $\gamma_{c1},\gamma_{c2}>0$ and $t \leq \frac{\alpha^2}{2\alpha^2\gamma_{c2}+1}$ the iterates of our algorithm converge linearly with a rate of 

$$
\delta=1-2\left[ 1+\frac{2\gamma_{c2}+\alpha^{-2}}{2\gamma_{c2}+\left\{\beta+\sqrt{p}(\beta-\alpha)\right\}^{-2}} \right]^{-1} <1.
$$
\end{theorem}

The proof for Theorem \ref{our_alg_conv} can be found in the Appendix. Theorem \ref{our_alg_conv} establishes the linear convergence of our proposed ISTA algorithm. Furthermore, these results show how $\gamma_{c2}$ influences the convergence of the algorithm, and the optimal solution bounds.  In particular, for a fixed $\gamma_{c1}$, as $\gamma_{c2}$ gets larger, the rate approaches 0. From a computational perspective, these results suggest that we could fix the step size parameter $t$ and avoid the backtracking line search when $p$ is large because $\alpha$ and $\gamma_{c2}$ can be calculated directly at each iteration.

\section{Gaussian graphical modeling simulation studies}
\label{ggsim}


\subsection{Overview}
\label{pcen_setup}
In our first set of simulations, we focus on both estimation accuracy and sparsity detection in Gaussian graphical modeling using PCEN.  We generated data from $C=4$ classes, where the $c$-th class is generated from ${\rm N}_p(0, \Omega_{*c}^{-1})$ and $p \in \{ 20, 50, 100\}$.  By construction, the sparsity patterns of $\Omega_{*1}$ and $\Omega_{*2}$ will be nearly equivalent; as will the sparisty patterns of $\Omega_{*3}$ and $\Omega_{*4}$. However, the sparsity patterns of $\Omega_{*1}$ and $\Omega_{*2}$ will be distinct from the sparsity patterns of $\Omega_{*3}$ and $\Omega_{*4}$. 

We compare two version of PCEN, PCEN-2 and PCEN-3 (i.e., \eqref{pcen} with $Q =2$ and $Q=3$, respectively) to the fused graphical lasso (FGL, \citet{danaher_2014}), graphical lasso with the same tuning parameter for all classes (Glasso), and two versions of the method proposed by \cite{saegusa2016} which we call LASICH-OR and LASICH-PR (denoting ``oracle" and ``practical", respectively). The method proposed by \citet{saegusa2016} requires the network information between the classes to be known before fitting the precision matrices (i.e., ``oracle" information), though it may be estimated using hierarchical clustering. In this simulation, a network where the edges are $\{ (1,2), (1,3), (2,4), (3,4)\}$ is used. The  difference between LASICH-OR and LASICH-PR is that LASICH-OR applies weights of $10^{-3}$ to the edges in the set $\{(1,3),(2,4)\}$ while LASICH-PR weights all edges equally.  In this way, LASICH-OR effectively uses the fact that sparsity patterns are shared between classes ${1,2}$ and ${3,4}$, which would be unknown in practice, while still allowing for overlap between the two sparsity patterns. Thus, this can be considered a ``best-case'' version of the HC-LASICH method. 
Tuning parameters for each of the methods are investigated  based a subset of  $(\lambda_1,\lambda_2) \in \{10^{-10},10^{-9.9},\ldots,10^{9.9}, 10^{10}\} \times  \{10^{-3},10^{-1},10^{1}\} $ unless otherwise specified.

To evaluate performance of each estimator, we use the sum of true positive (STP) across all $C$ classes, which we define as
$
\sum_{c=1}^4 \sum_{(j,k)}{\rm I}\left([\Omega_{*c}]_{jk}\neq 0 \cap [\hat{\Omega}_c]_{jk}\neq 0\right),
$
where $\hat{\Omega}_c$ is an estimate of $\Omega_{*c}$.
In addition, we also report the sum of the Frobenius norm squared error which is defined as
$
\sum_{c=1}^4 \|\Omega_{*c}-\hat{\Omega}_c\|_F^2.
$


In each replication, the training data consists of $n$ independent draws from each of the class distributions, i.e., $4n$ total realizations.  We investigate three different settings each based on Erdos-Renyi graphs. Throughout the settings we consider, we define $E(A,p)$ to be a $p \times p$ matrix where $A$ is an adjacency matrix associated with an Erdos-Renyi graph. To generate the elements of $E(A,p)$, we randomly assign each of the non-zero elements of $A$ a value from the set $(-0.7,-0.5)\cup (0.5,0.7)$.  Each off diagonal element is normalized by $1.5$ times the row sum of the matrix, and each diagonal element is set to 1. The matrix is then scaled such that the associated variance of each of the $p$ variables is 1. Further, we define $R(A,\Omega_*,V )$ to be a $p \times p$ matrix that is generated using the adjacency matrix $A$, such that nonzero elements are equal to the corresponding value in $\Omega_*$ plus a randomly selected value from the set $V$.  The off-diagonal elements are normalized by $1.5$ times the row sum of the matrix,  the diagonal elements are set to 1.  Finally, the entire matrix is normalized such that the variance of each variable is 1. Similar data generating mechanisms have been used in \cite{danaher_2014} and \cite{saegusa2016}.

\subsection{Two clusters, block Erdos-Reyni graphs}
\label{pcen_sim2}
We first compare PCEN-2 and PCEN-3 to competing methods under block Erdos-Reyni graphs. Each $(p,\lambda_1,\lambda_2)$ described in Section \ref{pcen_setup} is replicated 50 times with $n=200$. In this setting, we generate $\Omega_{*1}$ to be block diagonal with each block of size $p/2 \times p/2$. The first block is generated using $U=E(A_1,p/2)$, and the second is generated using $L=E(A_2,p/2)$ where $A_1$ and $A_2$ are adjacency matrices associated with independent Erdos-Reyni graphs with $p/2$ edges.  Using $\Omega_{*1}$, we generate $\Omega_{*2}$ such that it is block diagonal with block size $p/2 \times p/2$.  We define the upper block of $\Omega_{*2}$ as $R\left(A_3, L, (-.01,.01)\right)$, and the lower block to be $R\left(A_4, U, (-.01,.01)\right)$ where $A_3$ is the adjacency matrix $A_1$ with four edges removed.  Similarly $A_4$ is the adjacency matrix $A_2$ with 4 edges removed. Hence, $\Omega_{*1}$ and $\Omega_{*2}$ have nearly equivalent sparsity patterns minus eight nonzero entries in $\Omega_{*1}$ which are zero in $\Omega_{*2}$.

To generate $\Omega_{*3}$ we randomly select $p/2$ variables and define this set of variables as $s_1$ and define $s_2=\{1,\ldots,p\}\setminus s_1$.  The submatrix of $\Omega_{*3}$ corresponding to the indices in $s_1$ are generated such that $G=E(A_5, p/2)$ and submatrix of $\Omega_*$ corresponding to the indices in $s_2$ is generated such that $H=E(A_6,p/2)$, where $A_5$  and $A_6$ are independent Erdos-Renyi graphs with $p/2$ edges. The submatrices of $\Omega_{*4}$ corresponding to the indices in $s_1$ and $s_2$ are generated using $R\left(A_7, G, (-.01,.01)\right)$ and $R\left(A_8, H, (-.01,.01)\right)$, respectively.  The adjacency matrices $A_7$ and $A_8$ are the same as $A_5$ and $A_6$, respectively, with 4 randomly selected edges removed in each.  

The results in panels (a) and (b) of Figure \ref{fig:ggsim_2} are average log sum of squared Frobenius norm error and the average true positive rate as the number of non-zero elements in the estimated precision matrices varies with $p=100$. The results for the case of $p=20$ and $p=50$ can be found in the Supplementary Material. The results in panels (a) and (b) of Figure \ref{fig:ggsim_2} suggest that PCEN-2 can outperform as well or better than competitors in terms of Frobenius norm error and graph recovery. Notably, some versions of PCEN-2 outperform LASICH-OR both in terms of log sum of squared Frobenius norm and TPR, even though LASICH-OR knows the relations between precision matrices \textit{a priori}. 

In the Supplementary Material, we present additional simulation results examining the effect of sample size and $\lambda_2$ on the performance of PCEN-2. Briefly, as one would expect, as the sample size increases, the performance of PCEN-2 improves. In general, as $\lambda_2$ increases, the performance also improves.


\begin{figure}
\centering
\centerline{\includegraphics[width=.85\linewidth]{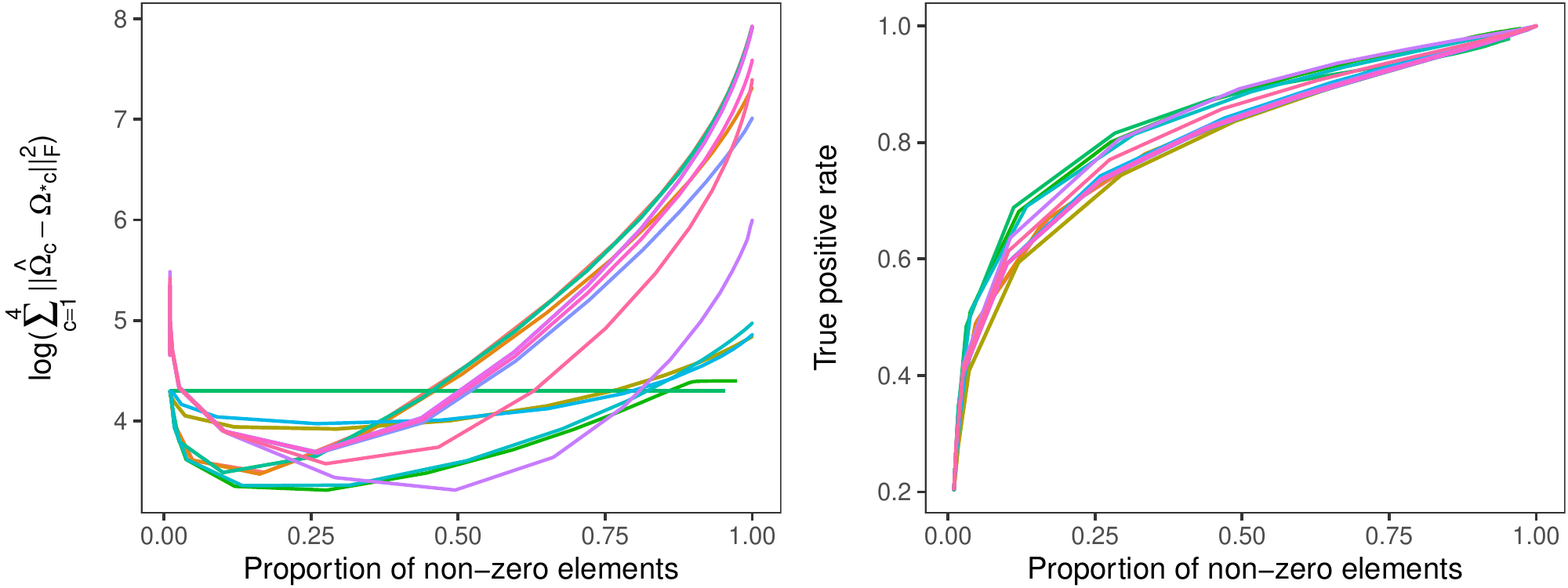}}
\centerline{\hfill\makebox{(a)}\hfill\quad\quad\makebox{(b)}\hfill}
\vspace{3pt}
\centerline{\includegraphics[width=.85\linewidth]{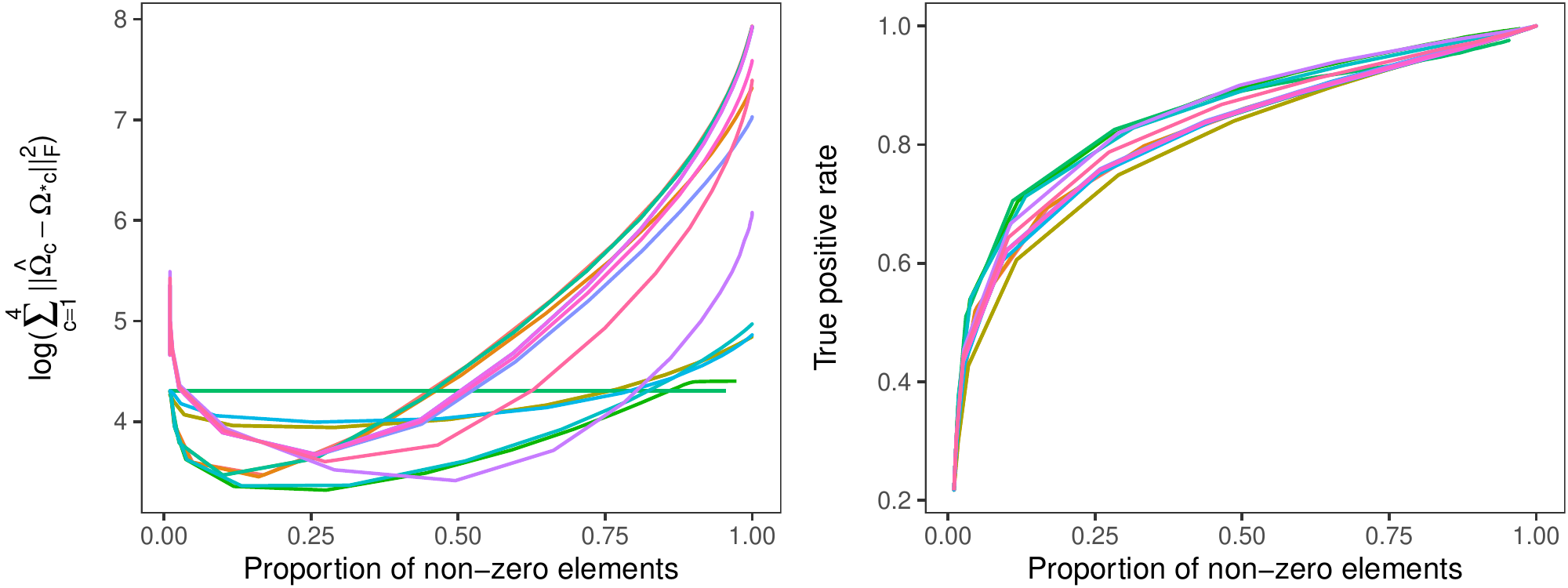}}
\centerline{\hfill\makebox{(c)}\hfill\quad\quad\makebox{(d)}\hfill}
\vspace{3pt}
\centerline{\includegraphics[width=.85\linewidth]{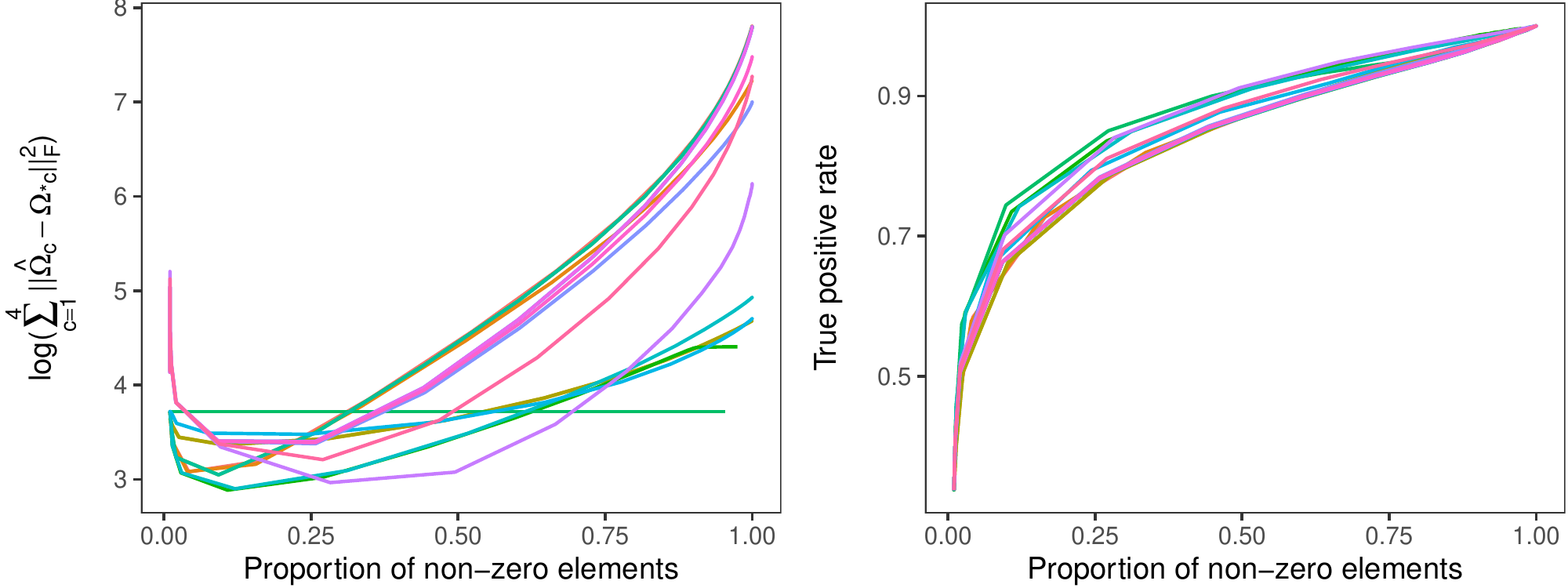}}
\centerline{\hfill\makebox{(e)}\hfill\quad\quad\makebox{(f)}\hfill}
\vspace{3pt}
\includegraphics[width=16cm]{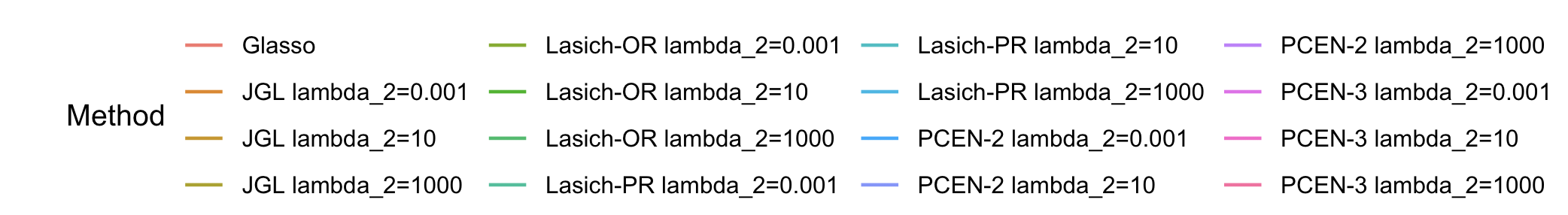}
\caption{Results for the simulation setting described in (a,b) Section \ref{pcen_sim2}, (c,d) Section \ref{pcen_sim3}, and (e,f) Section \ref{pcen_sim4} when $p=100$.  Each line represents the average of 50 replications of the denoted method when $\lambda_2$ is fixed, and $\lambda_1$ varies. }
\label{fig:ggsim_2}
\end{figure}

\subsection{Two clusters, block diagonal Erdos-Renyi graphs}
\label{pcen_sim3}

In contrast to the data generating models in Section \ref{pcen_sim2}, in these simulations we consider settings where all four precision matrices have a high degree of shared sparsity with high probability. We generate $\Omega_{*1}$ such that it is block diagonal with each block size  of $p/2 \times p/2$.  The first block is generated using $U=E(A_1,p/2)$, and the second block is generated from $L=E(A_2,p/2)$ where $A_1$ and $A_2$ are adjacency matrices associated with independent Erdos-Reyni graphs, with $p/2$ edges.  Using $\Omega_{*1}$ we generate $\Omega_{*2}$ such that it is block diagonal with block size $p/2 \times p/2$.  We define the upper  block of $\Omega_{*2}$ as $R\left(A_3, L, (-.01,.01)\right)$, and the lower block to be $R\left(A_4, U, (-.01,.01)\right)$ where $A_3$ is the adjacency matrix $A_1$ with four edges removed.  Similarly $A_4$ is the adjacency matrix $A_2$ with $.2p/2$ edges removed. Next, $\Omega_{*3}$ is generated in a similar way to $\Omega_{*1}$ and $\Omega_{*4}$ is generated from $\Omega_{*3}$ in the same fashion $\Omega_{*2}$ is generated from $\Omega_{*1}$.  By generating precision matrices in this way, entries not in the upper or lower block submatrix are zero in all four precision matrices.

The results in panels (c) and (d) Figure \ref{fig:ggsim_2} are average log sum of squared Frobenius norm error and the average true positive rate as the number of non-zero elements in the precision matrices varying with $p = 100$ and $n = 200$. The results for the case of $p=20$ and $p=50$ can be found in the Supplementary Material. These results show a similar pattern to the results from the simulation studies in Section \ref{pcen_sim2}. For large values of $\lambda_2$, PCEN-2 is competitive in Frobenius norm error and graph recovery with the all other methods, most notably LASICH-OR.  As mentioned, LASICH-OR has oracle knowledge of the true relationships between precision matrices, while PCEN is estimating the relationships as well as the precision matrices.

\subsection{Two clusters, block diagonal structures}
\label{pcen_sim4}
In the final setting, we again assume a data generating model where the four precision matrices are divided into two groups. We generate $\Omega_{*1}$ such that it is block diagonal with each block size of $p/2 \times p/2$.  The first block is generated using $U=E(A_1,p/2)$, and the second block is the identity matrix,  where $A_1$ is an adjacency matrix from an Erdos Renyi, with $p/2$ connections.  Using $\Omega_{*1}$ we generate $\Omega_{*2}$ such that it is block diagonal with block size $p/2 \times p/2$. We define the upper  block of $\Omega_{*2}$ as $R\left(A_3, L, (-.01,.01)\right)$, and the lower block to be the identity where $A_3$ is the adjacency matrix $A_1$ with four edges removed.  Next, $\Omega_{*3}$ is generated in a similar way to $\Omega_{*1}$ and $\Omega_{*4}$ is generated from $\Omega_{*3}$ in the same fashion $\Omega_{*2}$ is generated from $\Omega_{*1}$. 

The results in panels (e) and (f) Figure \ref{fig:ggsim_2} are average log sum of squared Frobenius norm error and the average true positive rate as the number of non-zero elements in the precision matrices varying with $p = 100$ and $n = 200$. The results for the case of $p=20$ and $p=50$ can be found in the Supplementary Material.  Results exhibit a similar pattern to the results displayed in Sections \ref{pcen_sim2} and \ref{pcen_sim3}.  For certain values of $\lambda_2$, PCEN-2 is competitive in estimation and graph recovery with the other methods, specifically LASICH-OR.  As $p$ increases, we see the estimation and graph recovery of PCEN decreases relative to LASICH-OR, but is still competitive with other competitors. Again, this can be attributed to LASICH-OR having oracle information and its use of the group penalty which exploits similar sparsity patterns across all precision matrices.

\section{Quadratic discriminant analysis simulations studies}\label{sim_3}



In this section, we study CRF as a method for fitting the QDA model. We generate data from $C=4$ classes, where predictors for the $c$-th class are generated from ${\rm N}_p(\mu_{*c}, \Sigma_{*c})$ with $p \in \{ 20, 50\}$.  The training data consists of 25 independent realizations from each class. Tuning parameters are selected using 5-fold cross-validation maximizing the validation likelihood (see Supplementary Material).  We measure classification accuracy to compare methods. To quantify classification accuracy, we generate an independent testing set consisting of 500 observations from each of the $C=4$ classes. 

In addition to CRF, RF, and RDA \citep{friedman_1989}, we include two methods which have oracle knowledge of the population parameters: Oracle, which uses $\Omega_{*c}$ and $\mu_{*c}$ in the classification rule; and TC (for ``true covariance"), which uses $\Omega_{*c}$ and the sample means in the classification rule. These oracle methods provide a benchmark for classification accuracy in these data. We omit the sparse methods discussed in Section \ref{ggsim} as we study a class of dense precision matrices in this particular simulation study.

We consider a situation where each of the two clusters has a distinct structure and precision matrices in both clusters are dense. For 100 independent replications, 
we generate $Z_{3} \in \Real^{100 \times p}$ where each row is an independent realization of ${\rm N}_p(0,I)$ and let $V_{3}$ be the right singular vectors of $Z_{3}$. We then let $\Sigma_{*1}=V_{3}^TH_3V_{3}$ and $\Sigma_{*2}=V_{3}^TH_4V_{3}$ where $H_3$ and $H_4$ are diagonal matrices with the $j$th element equal to $D(1000,100,j)$ and 
$D(999,99,j)$ respectively. Define the $(j,k)$th element of $(\Sigma_{*3})_{j,k}= 1(j=k)+ 0.45 \cdot 1(|j-k|=1)$ and $(\Sigma_{*4})_{j,k}= 1(j=k)+ \rho\cdot 1(|j-k|=1)$ where $1(\cdot)$ is the indicator function. We consider $(p,\rho) \in \{20,50\} \times \{0.40,0.47,0.50\}$. Finally, we set all elements of $\mu_{*1} = 20\log(p)/p$, $\mu_{*2} = -10\log(p)/p$, $\mu_{*3}=10\log(p)/p$, and 
$\mu_{*4} = -20\log(p)/p$.  A similar data generating model was used in \citet{price_2015}.
We expect CRF to perform well in this setting as it should be able to identify the distinct clusters, while RDA and RF implicitly assume similar structures across all precision matrices.

Table \ref{sim3_fig} presents a comparison of the classification error rate, and demonstrates that CRF out performs RDA and RF for every $(p,\rho)$ combination.  Interestingly, in the case that $p=20$, CRF performs nearly as well as TC, which uses the true covariance matrices. 

In the Supplementary Material, we provide additional simulation study settings and results under clustered, dense, and ill-conditioned precision matrices. 

\begin{table}
\caption{Results of simulation described in Section \ref{sim_3} comparing classification error rates and standard errors of CRF, RDA, RF and the two oracle methods for $(p,\epsilon)\in \{20,50\}\times\{1.0\}$.}
\label{sim3_fig}
\centering
\scalebox{.85}{
\fbox{
\begin{tabular}{c |c c c c c | c c c c c}
& \multicolumn{5}{c|}{$p=20$} & \multicolumn{5}{c}{$p=50$} \\
& RF & CRF & RDA & Oracle & TC & RF & CRF & RDA & Oracle & TC\\
\hline
\multirow{2}{*}{$\rho = 0.40$} &  0.237 & 0.106 & 0.237 & 0.015 & 0.108 & 0.238 & 0.130 & 0.238 &  0.005 & 0.075\\
 &  (0.001) & (0.003) & (0.001) &  (0.002) &  (0.013) & (0.001) & (0.002) & (0.001) & (0.000) & (0.010)\\
 \hline
 \multirow{2}{*}{$\rho = 0.47$} &  0.238 & 0.113  & 0.237 & 0.015 & 0.090  & 0.238 & 0.130 & 0.238 &  0.005 & 0.075\\
 & (0.002) & (0.004) & (0.001) & (0.002) & (0.012) & (0.001) & (0.002) & (0.001) & (0.002) & (0.010)\\
 \hline
 \multirow{2}{*}{$\rho = 0.50$} &  0.238 & 0.111 & 0.236 & 0.103 & 0.108 & 0.238 & 0.130 & 0.238 &  0.005 & 0.075\\
 &  (0.002) & (0.004) & (0.001) &  (0.002) &  (0.012) & (0.001) & (0.002) & (0.001) & (0.002) & (0.010)\\
\end{tabular}
}
}
\end{table}


\section{Data Examples}


\subsection{Gene Expression from Pulmonary Hypertension Patients}
\citet{pulmonary} collected gene expression profiles of 30 idiopathic pulmonary arterial hypertension patients (IPAH), 19 systemic sclerosis patients without pulmonary hypertension (SS w/o PH), 42 scleroderma-associated pulmonary arterial hypertension patients (SPAH), 8 systemic sclerosis patients with interstitial lung disease and pulmonary hypertension, and 41 healthy individuals, for a total of 140 individuals. The collected gene expression profiles consist of data from 49,576 probes. We scaled each probe to have a median of 256 and then performed a $\log_2$ transformation. Next, we scaled and centered the log transformed data to have mean zero and a standard deviation of one. Our analysis was focused on 132 individuals, excluding the 8 systemic sclerosis patients with interstitial lung disease and pulmonary hypertension, and 132 gene expression probes. The 132 probes we used were selected by running a one-factor ANOVA for each probe, using disease type as the factor, and then selecting the 132 probes with the smallest p-values. 

After this processing, we fit the PCEN model to the normalized data. The PCEN shrinkage tuning parameters were selected to promote sparsity in the graph and similarity between the graphs, similar to the procedure of \citet{danaher_2014}. We investigated the use of $Q=2$ and $Q=3$ clusters for these data. In both settings, PCEN was able to differentiate between the controls and patients with hypertension. In the case of two clusters, IPAH, SPAH and SS w/o PH are placed into a cluster while the control group is isolated in the second cluster. In the case of three clusters IPAH and SS w/o PH are placed into a cluster, while SPAH and the control group are both their own cluster of size one.  

\begin{figure}
\begin{center}
\includegraphics[width=4in]{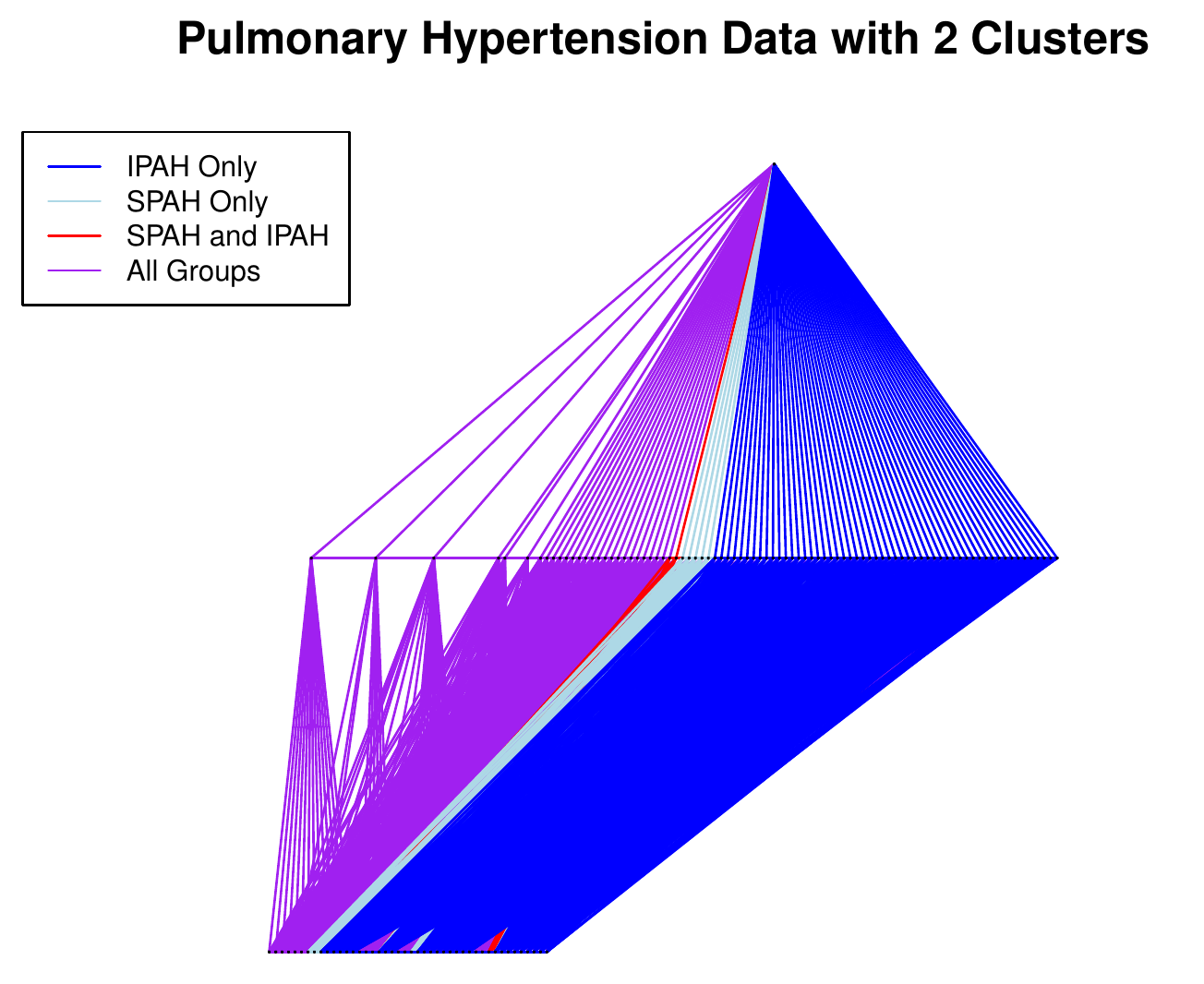}
\end{center}
\caption{Resulting network comparison from PCEN applied to the Pulmonary Hypertension Patients Data using $Q=2$ clusters.}
\label{2cluster_fig}
\end{figure}

Figure \ref{2cluster_fig} displays the corresponding network structures found using PCEN with $Q=2$ cluster. A similar plot for $Q=3$ is displayed in the Supplementary Material. In Figure \ref{2cluster_fig}, the blue edges represent probes that are related and were only found in patients diagnosed with IPAH, while light blue edges correspond to related probes found only in patients diagnosed with SPAH. Red edges denote relationships between probes that could be found in patients who were diagnosed with SPAH and those patients who were diagnosed with IPAH. Purple edges denote relationships between probes that could be found in patients who were diagnosed with SPAH and those patients who were diagnosed with IPAH and those who were diagnosed with SS w/o PH.  Table \ref{PH_table} presents the number of edges that appear in only IPAH and SPAH, and then the edges that are present in both graphs.

At first inspection, the results between the cases of $Q=2$, shown in Figure \ref{2cluster_fig}, and $Q=3$, presented in the Supplementary Material, appear similar, but there are very notable differences.  When $Q=3$ and SPAH belongs to its own cluster, we see that the number of shared edges between all groups decreases, which is expected. The other differences, which are quantified in Table \ref{PH_table}, can be attributed to the changing cluster structure and may have important biological implications. 


\begin{table}
\caption{A comparison of network differences produced by PCEN using 2 and 3 clusters for the Pulmonary Hypertension Patients Data.  The values in the table are the number of edges that are present only in IPAH, SPAH, or are present in both.}
\label{PH_table}
\begin{center}
\fbox{
\begin{tabular}{l|cccc|c}
& IPAH & SPAH & IPAH  and SPAH & All Groups &  Total \\ 
\hline 
2 Clusters & 1636  & 354  & 96  &  2114 & 4200\\ 
3 Clusters & 1431 & 467   & 143  & 1918 & 3959\\ 
\end{tabular} 
}
\end{center}
\end{table}

\begin{table}
\caption{Classification results from the Libras Data example.  }
\label{libra_table}
\begin{center}
\fbox{
\begin{tabular}{c|c|c|c|c}
Method & CRF & RDA & Ridge & Ridge Fusion \\ 
\hline 
Error rate & 13/60 & 20/60 & 51/60 & 51/60 \\ 
\end{tabular} 
}
\end{center}
\end{table}

\subsection{Libras Data Example}
To further demonstrate the useful of our proposed method,  we apply CRF to a classification problem based on the Libras data set from the UCI Machine Learning repository \citep{UCI}.  These data contain 15 classes, each of which corresponds to a videoed hand movement of Brazilian sign language. Each hand movement was recorded at 45 distinct time frames and the coordinates on an $x-y$ plane were documented, which results in 90 predictor variables for the hand movement. Each of the 15 classes has 24 observations for a total of 360 observations.  Training was done using 20 randomly selected observations from each class, and testing was done on the four remaining observations. Our test and training sets are available in the Supplementary Material. We compare four methods: CRF, RF, ridge penalized normal likelihood precision matrix estimation, and RDA. The ridge penalized normal likelihood precision matrix estimator is equivalent to CRF with $\lambda_2 = 0$.  Tuning parameters were selected by five-fold cross-validation maximizing a validation likelihood for all likelihood based methods. In the case of CRF, the number of clusters was chosen from the set of integers ranging from 2 to 10.  For the non-likelihood method, RDA, we selected tuning parameters by five-fold cross validation minimizing the misclassification rate. 


Table \ref{libra_table} contains the classification error rate for each of the five methods on the testing data. The CRF method out performs the other methods with regards to classification rate and detects two clusters.  Further investigation shows that for CRF 9 out of 15 of the classes had a CER of 0.

\section{Acknowledgments}
This work has been supported in part by NSF MRI Award \# 11726534.  

\bibliographystyle{jcgs}
\bibliography{cluster_da}

\section{Appendix}




\subsection{Proof of Lemma 1}
\label{lemma1proof}


In this section we show that the optimal solution of \eqref{reform} is contained on a compact subset of $\Sp_+^p$. To gain a deeper understanding of the solution we obtain the dual form of \eqref{reform}.  Define $Z_c$ to be a symmetric $p \times p$ matrix, then \eqref{reform} can be rewritten as 

\begin{equation*}
\min_{\Omega_c \in \Sp_+^p} \tr(\tilde{S}_c\Omega_c)-\log\det(\Omega_c)+\gamma_{c1}\max_{|vec(Z_c)|_\infty <1} \tr(Z_c\Omega_c)+
\gamma_{c2}|\Omega_c|_2^2.\\
\end{equation*}
Just as in \cite{banerjee08} we exchange the max and min to obtain the dual problem,
\begin{equation}
\label{dual_1}
\max_{|vec(Z_c)|_\infty <1} \min_{\Omega_c \in \Sp_+^p}  \tr((\tilde{S}_c+\gamma_{c1} Z_c)\Omega_c)-\log\det(\Omega_c)+\gamma_{c2}|\Omega_c|_2^2.
\end{equation}
Notice that the optimization problem in \eqref{dual_1} with respect to $\Omega_c$ is just a ridge penalized precision matrix estimation problem with tuning parameter $\gamma_2$ which was first investigated  by \citep{witten09}.  Define 
$$
Q(A,\eta)=\argmin_{\Theta \in \Sp_+^p} \tr(A\Theta)-\log\det(\Theta)+\eta|\Theta|_2^2,
$$
and $\dot{\Omega}_{Z_c}=Q(\tilde{S}_c+\gamma_{c1}Z_c,\lambda_2)$, then the dual problem can then be rewritten as
\begin{equation}
\label{dual_2}
\max_{|vec(Z_c)|_\infty <1}  \tr((\tilde{S}_c+\gamma_{c1} Z_c)\dot{\Omega}_{Z_c})-\log\det(\dot{\Omega}_{Z_c})+\gamma_{c2}|\dot{\Omega}_{Z_c}|_2^2.
\end{equation}




Now we are able to show the result.

\begin{proof}
Define 
$$
q(a,\eta)=\frac{-a+\sqrt{a^2+4\eta}}{2\eta}.
$$ 
Note that $q(a,\eta)>0$ for all $a \in \Real$ when $\eta>0$ and given some $b \in \Real$ such that $b<a$ then  $q(a,\eta)>q(b,\eta)>0$. Let $\widehat{Z}_c$ be the solution to \eqref{dual_2}. Then we are able to rewrite $\Omega_c^{*}=\dot{\Omega}_{\widehat{Z}_c}=Q(\tilde{S}_c+\gamma_{c1}\widehat{Z}_c,\gamma_{c2})=V\widehat{D}V^T$, where $V$ is a matrix of the eigenvectors of $\tilde{S}_c+\gamma_{c1}\widehat{Z}_c$, and $\widehat{D}$ is a diagonal matrix where the inverse of the $j$th diagonal element is equal to $q\left(\rho_j(\tilde{S}_c+\gamma_{c1}\widehat{Z}_c),\gamma_{c2}\right)$ \citep{witten09}. To complete the proof all that is left to do is bound the cases of $j=1$ and $j=p$.

Weyl's Theorem provides the inequalities
$$
\rho_{p}(\tilde{S}_c)+\gamma_{c1}\rho_{p}(\widehat{Z}_c)\leq \rho_{p}(\tilde{S}_c+\gamma_{c1}\widehat{Z}_c),
$$
and 
$$
\rho_{1}(\tilde{S}_c+\gamma_{c1}\widehat{Z}_c)\leq \rho_{1}(\tilde{S}_c)+\gamma_{c1}\rho_{1}(\widehat{Z}_c).
$$

Finally note the inequality 
$$
-p\leq -\|\widehat{Z}_c\|_2  \leq \rho_p(\widehat{Z}_c)\leq \rho_1(\widehat{Z}_c)\leq \|\widehat{Z}_c\|_2 \leq p.
$$

Combining all results this leads to the inequality
$$
0<q(\rho_{p}(\tilde{S}_c)-\gamma_{c1}p,\gamma_{c2})\leq q(\rho_{p}(\tilde{S}_c)+\gamma_{c1}\rho_{p}(\widehat{Z}_c),\gamma_{c2})\leq q(\rho_{1}(\tilde{S}_c)+\gamma_{c1}\rho_{1}(\widehat{Z}_c),\gamma_{c2}) \leq q(\rho_{1}(\tilde{S}_c)+\gamma_{c1}p,\gamma_{c2})<\infty.
$$
These resulting bounds are for eigenvalues of $(\Omega^{*})^{-1}$, and by inverting the bounds you obtain the results of the lemma.

\end{proof}

\subsection{Lipschitz Continuity of the $\bigtriangledown f(\Omega)$}

\begin{proof}
Assume $0<aI \preceq \Omega_A,\Omega_B \preceq bI$, for some $\Omega_A,\Omega_B \in \Sp_+^p$.  By Lemma 2 of \citet{rolfs12} we have that 
$$
\frac{1}{b^2}\|\Omega_A-\Omega_B\|_2 \leq \| \Omega_A^{-1}-\Omega_B^{-1}\|_2 \leq \frac{1}{a^2}\|\Omega_A-\Omega_B\|_2.
$$

Finally we have

\begin{eqnarray*}
\| \bigtriangledown f(\Omega_A)- \bigtriangledown f(\Omega_B)\|_F &=& \|\Omega_B^{-1}-\Omega_A^{-1} +2\gamma_{c2}(\Omega_A-\Omega_B)\|_F\\
&\leq& \sqrt{p}\|\Omega_B^{-1}-\Omega_A^{-1} +2\gamma_{c2}(\Omega_A-\Omega_B)\|_2\\
&\leq& \sqrt{p}\|\Omega_B^{-1}-\Omega_A^{-1}\|_2 +2\sqrt{p}\gamma_{c2}\|\Omega_A-\Omega_B\|_2\\
&\leq& \frac{\sqrt{p}}{a^2}\|\Omega_A-\Omega_B\|_2+2\sqrt{p}\gamma_{c2}\|\Omega_A-\Omega_B\|_2\\
&\leq&\left( \frac{\sqrt{p}}{a^2}+ 2\sqrt{p}\gamma_{c2}\right)\|\Omega_A-\Omega_B\|_F.
\end{eqnarray*}
\end{proof}

\subsection{Proof of Theorem \ref{our_alg_conv}}

First, we will provide a general result that we use to establish the linear convergence rate of our algorithm.  
\begin{lemma}
\label{diff_bound}
Assume that iterates of the algorithm proposed satisfy $aI \preceq \Omega_c^{(k)} \preceq bI$ for all $k$ and some fixed constants $0<a<b$. If $t\leq \frac{a^2}{2\alpha^2\gamma_{c2}+1}$ then:
\begin{enumerate}
\item $$
\| \Omega_c^{(k+1)}-\Omega_c^{*}\|_F\leq \max \left(|m_t-\frac{t}{a^2}|,|m_t-\frac{t}{b^2}|)\|\right)\| \Omega_c^{(k)}-\Omega_c^{*}\|_F,
$$
where $m_t=1-2t\gamma_{c2}$. 
\item The step size $t$ that will lead to the optimal worst-case bound is $t_w=\frac{2}{4\gamma_{c2}+b^{-2}+a^{-2}}$.
\item The optimal worst case bound is 
$$
1-\frac{2}{1+\frac{2\gamma_{c2}+a^{-2}}{2\gamma_{c2}+b^{-2}}}<1.
$$
\end{enumerate}
\end{lemma}

We present the full proof in section \ref{diff_b_sec} but there are a few things to note.  First is that if $\gamma_{c2}=0$ then this result obtains the bounds of \cite{rolfs12}.  Second is the fact that as $\gamma_{c2}$ approaches $\infty$ the optimal worst case bounds approach 0. Finally, as $\gamma_{c2}$ gets larger the maximum step size that is applicable also approaches 0.

\subsubsection{Proof of Lemma \ref{diff_bound}}
\label{diff_b_sec}

Our proof strategy is similar to that of  \cite{rolfs12} but there are differences due to the ridge penalty.
\begin{proof}
Recall that $\Omega_c^*=\Sft(\Omega_c^{*} -t(\tilde{S}_c-(\Omega_c^{*})^{-1}+2\gamma\Omega_c^{*}),t\gamma_{c1})$. By the definitions of $\Omega_c^{(k+1)}$ and $\Omega_c^{(k+\frac{1}{2})}$ and Lemma 2.2 from \citet{combettes05}
\begin{eqnarray*}
\|\Omega_c^{(k+1)}-\Omega_c^*\|_F&=&\|\Sft(\Omega_c^{(k+\frac{1}{2})},t\gamma_{c1})-\Sft(\Omega_c^{*} -t(\tilde{S}_c-(\Omega_c^{*})^{-1}+2\gamma\Omega_c^{*}),t\gamma_{c1})\|_F\\
&\leq& \|\Omega_c^{(k+\frac{1}{2})}-(\Omega_c^{*} -t(\tilde{S}_c-(\Omega_c^{*})^{-1}+2\gamma\Omega_c^{*}))\|_F\\
&=& \| \Omega_c^{(k+\frac{1}{2})}+t\tilde{S}_c - [(1-2t\gamma)\Omega_c^*+t(\Omega_c^*)^{-1}]\|_F\\
&=& \| \Omega_c^k - t(\tilde{S}_c-(\Omega_c^{(k)})^{-1}+2\gamma_{c2}\Omega_c^{(k)})+t\tilde{S}_c - [(1-2t\gamma)\Omega_c^*+t(\Omega_c^*)^{-1}]\|_F\\
&=&\|\left(\Omega_c^{(k)}-t\left(2\gamma_{c2}\Omega_c^{(k)}-(\Omega_c^{(k)})^{-1}\right)\right) -\left(\Omega_c^{*}-t\left(2\gamma_{c2}\Omega_c^{*}-(\Omega_c^{*})^{-1}\right)\right)\|_F \\
&=& \| [(1-2t\gamma)\Omega_c^{(k)} + t(\Omega_c^{(k)})^{-1}] - [(1-2t\gamma)\Omega_c^{*} + t(\Omega_c^{*})^{-1}] \|_F.
\end{eqnarray*}


If $h: U \subset \Real^{p^2} \rightarrow \Real^m$ is a  differentiable mapping with Jacobian $J_h$, $x,y \in U$ and $vx+(1-v)y\in U$ for all $v \in [0,1]$, then
$$
\|h(x)-h(y)\|_2 \leq \sup_{v \in [0,1]} \|\{J_h(vx+(1-v)y)\|x-y\|_2\}. 
$$
Recall $m_t=(1-2t\gamma_{c2})$ and define
$$
h_{\gamma_{c1},\gamma_{c2}}(\VEC(\Omega_c))=m_t \VEC(\Omega_c)+t\VEC(\Omega_c^{-1}).
$$
Note that, 
$$
J_{h_{\gamma_{c1},\gamma_{c2}}}(\Omega_c)=m_tI_{p^2}-t\Omega_c^{-1}\otimes\Omega_c^{-1}.
$$

For $v \in [0,1]$ let 
$$
H_{k,v}= \VEC\left({v\Omega_c^{(k)}+(1-v)\Omega_c^{*}}\right),
$$

it follows that  
$$
\|h_{\gamma_{c1},\gamma_{c2}}(\Omega_c^{(k)})-h_{\gamma_{c1},\gamma_{c2}}(\Omega_c^{*})\|_2\leq \sup_{v \in [0,1]}\{\| m_tI_{p^2}-tH_{v,k}^{-1}\otimes H_{v,k}^{-1}\|_2\}\|\Omega_c^{(k)}-\Omega_c^{*}\|_F.
$$

Therefore, for any value of $k$ and $v$ 
$$
\min\{\rho_p(\Omega_c^{(k)}),\rho_p(\Omega_c^{*})\} \leq \rho_p(H_{k,v})\leq \rho_1(H_{k,v})\leq \max\{\rho_1(\Omega_c^{(k)}),\rho_1(\Omega_c^{*})\}.
$$

Combining these results we obtain

$$
\sup_{v \in [0,1]}\{\| m_tI_{p^2}-tH^{-1}\otimes H^{-1}\|_2\} \leq \max\{|m_t-\frac{t}{b^2}|, |m_t-\frac{t}{a^2}|\},
$$
which proves part 1 of Lemma \ref{diff_bound}.  

We can further show that the algorithm converges linearly if 
$$
s(t)=\max\{|m_t-\frac{t}{b^2}|, |m_t-\frac{t}{a^2}|\} \in (0,1), \, \forall k. 
$$
The minimum of $s(t)$ is obtained at 
$$
t_w=\frac{2}{4\gamma_{c2}+b^{-2}+a^{-2}},
$$
and then evaluating $s(t_w)$ completes the result. 

\end{proof}

\subsubsection{Proof of Lemma \ref{iterate_bound}}
\label{iter_bound_sec}
In this section we assume that the eigenvalues of $\Omega_c^{(k)}$ are bounded for all $k$ 
and recall that


\begin{equation}
\label{half}
\Omega_{c}^{(k+\frac{1}{2})}=\Omega_c^{(k)}-t\left(\tilde{S}_c-(\Omega_c^{(k)})^{-1}+2\gamma_{c2}\Omega_c^{(k)}\right).
\end{equation}

\begin{lemma}
\label{half_step_lemma}
Assume $0<a<b$ are known such that $aI \preceq \Omega_c^{(k)} \preceq bI$, and that $t>0$.  Then the eigenvalues of $\Omega_c^{(k+\frac{1}{2})}$, which is defined by \eqref{half}, satisfy:

\begin{align*}
   \rho_p(\Omega_c^{(k+\frac{1}{2})}) \geq 
    \quad
    \begin{cases}
    \sqrt{\frac{t}{1-2t\gamma_{c2}}} -\frac{t}{\sqrt{\frac{t}{1-2t\gamma_{c2}}}}-t\rho_1(\tilde{S}_c),  \,\,\,\,\, \text{if } a \leq \sqrt{\frac{t}{1-2t\gamma_{c2}}} \leq b \\
   \min(m_ta+\frac{t}{a}, m_tb+\frac{t}{b})-t\rho_1(\tilde{S}_c)\,\,\,\, \text{otherwise}
    \end{cases}
    \quad
\end{align*}

and

$$
\rho_1(\Omega_c^{(k+\frac{1}{2})})\leq   \max(m_ta+\frac{t}{a}, m_tb+\frac{t}{b})-t\rho_p(\tilde{S}_c).
$$

\end{lemma}

\begin{proof}
Define the spectral decomposition of $\Omega_c^{(k)}=UDU^T$, then

\begin{eqnarray*}
\Omega_c^{(k+\frac{1}{2})}&=&\Omega_c^{(k)}-t\left(\tilde{S}_c-(\Omega_c^{(k)})^{-1}+2\gamma_{c2}\Omega_c^{(k)}\right)\\
&=&UDU^T-t(\tilde{S}_c-UD^{-1}U^T+2\gamma_{c2}UDU^T)\\
&=&U(D-t(U^T\tilde{S}_cU-D^{-1}+2\gamma_{c2}D))U^T.
\end{eqnarray*}

Next, by Wyel's Theorem it follows that 

$$
\rho_p(\Omega_c^{(k+\frac{1}{2})})\geq \rho_p(\Omega_c^{(k)})-t\left(2\gamma_{c2}\rho_p(\Omega_c^{(k)})-\frac{1}{\rho_p(\Omega_c^{(k)})}+\rho_1(\tilde{S}_c)\right),
$$

and 

$$
\rho_1(\Omega_c^{(k+\frac{1}{2})})\leq \rho_1(\Omega_c^{(k)})-t\left(2\gamma_{c2}\rho_1(\Omega_c^{(k)})-\frac{1}{\rho_1(\Omega_c^{(k)})}+\rho_p(\tilde{S}_c)\right).
$$

The function

$$
r(x)=m_tx+\frac{t}{x},\,\, \, a \leq x \leq b,
$$
has a  global minimum at $x_w=\sqrt{\frac{t}{1-2t\gamma_{c2}}}$.  Thus, using arguments similar to \cite{rolfs12} proof of Lemma 4, we have that 

\begin{align*}
   \rho_p(\Omega_c^{(k+\frac{1}{2})}) = 
    \quad
    \begin{cases}
    \sqrt{\frac{t}{1-2t\gamma_{c2}}} -\frac{t}{\sqrt{\frac{t}{1-2t\gamma_{c2}}}}-t\rho_1(\tilde{S}_c),  \,\,\,\,\, \text{if } a \leq \sqrt{\frac{t}{1-2t\gamma_{c2}}} \leq b \\
   \min(m_ta+\frac{t}{a}, m_tb+\frac{t}{b})-t\rho_1(\tilde{S}_c)\,\,\,\, \text{otherwise}
    \end{cases}
    \quad
\end{align*}

and 

$$
\rho_1(\Omega_c^{(k+\frac{1}{2})})\leq   \max(m_ta+\frac{t}{a}, m_tb+\frac{t}{b})-t\rho_p(\tilde{S}_c),
$$

which obtain the bounds.
\end{proof} 

Next we need to show that when the full step is taken by soft thresholding the eigenvalues are bounded.

\begin{lemma}
\label{min_bound_equiv_lemma}
Assume $0<a<b$ and $t, m_t>0$ then $\min(m_ta+\frac{t}{a}, m_tb+\frac{t}{b})=m_ta+\frac{t}{a}$ if and only if $ t \leq \frac{ab}{1+2\gamma_{c2} ab}$
\end{lemma}

\begin{proof}
Using the assumptions we have that 

\begin{eqnarray*}
m_ta+\frac{t}{a} \leq m_tb+\frac{t}{b} &\Leftrightarrow& t(\frac{1}{a}-\frac{1}{b}) \leq m_t(a-b)\\
&\Leftrightarrow& t \leq m_tab \Leftrightarrow t \leq \frac{ab}{1+2\gamma_{c2} ab}.
\end{eqnarray*}

\end{proof}

Next, for the sake of convenience, we restate Lemma 6, a useful result on soft thresholding, from the supplementary material of \citet{rolfs12}.

\begin{lemma}
\label{thresh_lemma}
Let $A$ be a symmetric $p \times p$ matrix.  Then the soft thresholded matrix $Sft(A,\delta)$ satisfies 
$$
\rho_p(A)-p\delta\leq \rho_p(\Sft(A,\delta))
$$ 

More over  the soft thresholded matrix is positive definite if $\rho_p(A)>p\delta $\citep{rolfs12}.
\end{lemma}
\begin{proof}
Proof is in the supplementary material of \citet{rolfs12}.
\end{proof}

\begin{lemma}
\label{lower_bound_lemma}
Let $\gamma_{c1}>0$ and $\alpha$ be the same as defined in Lemma \ref{optim_bounds}.  Assume $\alpha < b'$ and $\alpha I \preceq \Omega_c^{(k)} \preceq b'I$ and recall that 
$$
\Omega_c^{(k+1)}=\Sft( \Omega_c^{(k+\frac{1}{2})},t\gamma_{c1}),
$$
where $\Omega_c^{(k+\frac{1}{2})}$ is defined by \eqref{half}.  Then for every $0<t \leq \frac{\alpha^2}{2+\gamma_{c2}\alpha^2+1}$, then $\alpha I \preceq \Omega_c^{(k+1)}$. 
\end{lemma}

\begin{proof}
Lemma \ref{min_bound_equiv_lemma} gives us
$$
\min(m_t\alpha+\frac{t}{\alpha},m_tb+\frac{t}{b'})=m_t\alpha+\frac{t}{\alpha}.
$$
since $t\leq \frac{\alpha^2}{2\gamma_{c2}\alpha^2+1} \leq \frac{\alpha b'}{2\gamma_{c2}\alpha b' +1}$.
 Note that $0<t \leq \frac{\alpha^2}{2\gamma_{c2}\alpha^2+1}$ guarantees that $\sqrt{\frac{t}{1-2t\gamma_{c2}}}\leq \alpha$.  Therefore by Lemma \ref{half_step_lemma}
 
 $$
 \rho_p(\Omega_c^{(k+\frac{1}{2})})\geq m_t\alpha+\frac{t}{\alpha}-t\rho_1(\tilde{S}_c).
 $$
 
We continue by applying Lemma \ref{thresh_lemma} to $\Omega_c^{(k+1)}$ where we obtain

\begin{eqnarray*}
\rho_p(\Omega_c^{(k+1)})&\geq& \rho_p(\Omega_c^{(k+\frac{1}{2})})-p\gamma_{c1}t\\
&\geq&  m_t\alpha+\frac{t}{\alpha}-t\rho_1(\tilde{S}_c)-p\gamma_{c1}t.
\end{eqnarray*}

Therefore, we have that $\alpha I \preceq \Omega_c^{(k+1)}$ when 
$$
m_t\alpha+\frac{t}{\alpha}-t\rho_1(\tilde{S}_c)-p\gamma_{c1}t \geq \alpha,
$$
or equivalently 
$$
-2\gamma_{c2}t\alpha+\frac{t}{\alpha}-t\rho_1(\tilde{S}_c)-p\gamma_{c1}t\geq 0. 
$$

Since $t>0$ we may reorganize this a final time as 

$$
-2\gamma_{c2}\alpha+\frac{1}{\alpha}-\left(\rho_1(\tilde{S}_c)+p\gamma_{c1}\right)\geq 0.
$$

Solving for $\alpha$ we have that $\alpha I \preceq \Omega_c^{(k+1)}$ if
$$
\alpha \leq \frac{1}{q(\rho_1(\tilde{S}_c)+\gamma_{c1}p,2\gamma_{c2})},
$$
which holds by Lemma \ref{optim_bounds}.
\end{proof}

\begin{lemma}
\label{upper_bound_lemma}
Let $\alpha$ be the same as in Lemma \ref{optim_bounds} and $t \leq \frac{\alpha^2}{2+\gamma_{c2}\alpha^2+1}$.  The the proposed algorithm iterates satisfy $\Omega_c^{(k)} \preceq b'I$ for all $k$ where $b'=\|\Omega_c^{*}\|_2+\|\Omega_c^{(0)}-\Omega^{*}\|_F$.
\end{lemma}

\begin{proof}
Using results from Lemma \ref{optim_bounds} and Lemma \ref{lower_bound_lemma} we have that 
$$
\Lambda_k^{+}=\max\left( \rho_1(\Omega_c^{(k)}),\rho_1(\Omega_c^{*})\right)>\Lambda_k^{-}=\min\left( \rho_p(\Omega_c^{(k)}),\rho_p(\Omega_c^{*})\right) \geq \alpha^2.
$$

Since $t \leq \frac{\alpha^2}{2+\gamma_{c2}\alpha^2+1}$, 
$$
\max\left\lbrace |1-\frac{t}{b^2}|,|1-\frac{t}{a^2}|\right\rbrace \leq 1.
$$
Next, by applying Theorem \ref{iterate_bound} 

$$
\|\Omega_c^{(k)}-\Omega_c^{*}\|_F\leq \|\Omega_c^{(k-1)}-\Omega_c^{(k)}\|_F.
$$

Finally we have that,
$$
\|\Omega_c^{(k)}\|_2-\|\Omega_c^{*}\|_2 \leq \|\Omega_c^{(k)}-\Omega_c^{*}\|_2\leq \|\Omega_c^{(0)}-\Omega_c^{*}\|_F.
$$

Finally we obtain the bound,
$$
\rho_1(\Omega_c^{(k)})\leq \|\Omega_c^{*}\|_2+\|\Omega_c^{(0)}-\Omega_c^{*}\|_F.
$$
\end{proof}

Finally we will formally state the proof of lemma \ref{iterate_bound}.

\begin{proof}
Applying the results of Lemma \ref{lower_bound_lemma} and Lemma \ref{upper_bound_lemma} we have that 

$$
\alpha I \preceq \Omega_c^{(k)} \preceq b'I,
$$
and 
\begin{eqnarray*}
b'&\leq&\|\Omega_c^{*}\|_2+\sqrt{p}\|\Omega_c^{(0)}-\Omega_c^{*}\|_2 \\
&\leq& \beta +\sqrt{p}(\beta-\alpha).
\end{eqnarray*}
\end{proof}

\end{document}